\documentclass[times,a4paper,twocolumn,final,authoryear]{elsarticle}

\usepackage{style}
\usepackage{framed,multirow}

\usepackage{amssymb}
\usepackage{latexsym}
\usepackage{amsthm}%

\usepackage{url}
\usepackage{xcolor}
\definecolor{newcolor}{rgb}{.8,.349,.1}

\usepackage[titletoc]{appendix}
\usepackage{url}

\usepackage{amsmath,amssymb}
\usepackage[utf8]{inputenc}
\usepackage{graphicx}%
\usepackage{braket}
\usepackage{tikz}
\usetikzlibrary{automata,arrows,arrows.meta,positioning,calc,shapes,trees}

\usepackage{bbm}%
\usepackage{subcaption}
\usepackage{relsize}
\usepackage{float}
\floatstyle{plaintop}%
\restylefloat{table}
\usepackage{pgffor}

\usepackage[space]{grffile} %

\usepackage{datatool}

\usepackage{csvsimple}
\usepackage{xspace}
\usepackage{etoolbox} %
\usepackage{mathtools} %

\usepackage{xargs}                      %
\usepackage{hyperref}
\hypersetup{
    colorlinks,
    citecolor=black,
    filecolor=black,
    linkcolor=black,
    urlcolor=black
}

\newtheorem{mathdef}{Definition}

\usepackage{xhfill}

\usepackage[breakable, theorems, skins]{tcolorbox}

\usepackage[normalem]{ulem} %

\newtheorem{theorem}{Theorem}

\newtheorem{lemma}[theorem]{Lemma}

\theoremstyle{definition}
\newtheorem{exmp}{Example}[section]

\makeatletter
\newcommand{\distas}[1]{\mathbin{\overset{#1}{\kern\z@\sim}}}%
\newsavebox{\mybox}\newsavebox{\mysim}
\newcommand{\distras}[1]{%
  \savebox{\mybox}{\hbox{\kern3pt$\scriptstyle#1$\kern3pt}}%
  \savebox{\mysim}{\hbox{$\sim$}}%
  \mathbin{\overset{#1}{\kern\z@\resizebox{\wd\mybox}{\ht\mysim}{$\sim$}}}%
}
\makeatother

\DeclareMathOperator{\EX}{\mathbb{E}}%
\DeclareMathOperator{\Indicator}{\mathbbm{1}}%

\DeclarePairedDelimiterX{\infdivx}[2]{(}{)}{%
  #1\;\delimsize\|\;#2%
}

\newcommand{\slfrac}[2]{\left.#1\middle/#2\right.} %

\begin{document}

\twocolumn[
  \begin{@twocolumnfalse}

\begin{center}

{\LARGE \textbf{Do Compressed Representations Generalize Better?}}

\vskip30pt

Hassan Hafez-Kolahi ~~ Shohreh Kasaei* ~~ Mahdieh Soleymani-Baghshah \\
\vskip10pt
  Sharif University of Technology,
  Tehran, 
  Iran\\
\texttt{hafez@ce.sharif.edu, skasaei@sharif.edu, soleymani@sharif.edu} 

\end{center}

\section*{Abstract}

One of the most studied problems in machine learning is finding reasonable constraints that guarantee the generalization of a learning algorithm. These constraints are usually expressed as some \emph{simplicity} assumptions on the target. For instance, in the Vapnik-Chervonenkis (VC) theory the space of possible hypotheses is considered to have a limited VC dimension and in kernel methods there are assumptions on the spectrum of the operator in the Hilbert space. One way to formulate the simplicity assumption is via information theoretic concepts. In this paper, the constraint on the entropy $H(X)$ of the input variable $X$ is studied as a simplicity assumption. It is proven that the sample complexity to achieve an $\epsilon$-$\delta$ Probably Approximately Correct (PAC) hypothesis is bounded by $\frac{2^{\slfrac{6H(X)}{\epsilon}}+\log{\frac{1}{\delta}}}{\epsilon^2}$ which is sharp up to the $\frac{1}{\epsilon^2}$ factor. Morever, it is shown that if a feature learning process is employed to learn the compressed representation from the dataset, this bound no longer exists. These findings have important implications on the Information Bottleneck (IB) theory which had been utilized to explain the generalization power of Deep Neural Networks (DNNs), but its applicability for this purpose is currently under debate by researchers. In particular, this is a rigorous proof for the previous heuristic that compressed representations are exponentially easier to be learned. However, our analysis pinpoints two factors preventing the IB, in its current form, to be applicable in studying neural networks. Firstly, the exponential dependence of sample complexity on $\slfrac{1}{\epsilon}$, which can lead to a dramatic effect on the bounds in practical applications when $\epsilon$ is small. Secondly, our analysis reveals that arguments based on input compression are inherently insufficient to explain generalization of methods like DNNs in which the features are also learned using available data.

\vskip 0pt
\noindent \textbf{Keywords:} Compressed Representation; Generalization Bound;  Information Bottleneck.

\vskip 20pt

\end{@twocolumnfalse}
]

\section{Introduction}
The main objective of \emph{learning} is to develop algorithms which can learn general patterns by using a finite number of samples drawn from a target distribution.
The \emph{"no free lunch"} theorem states that if there is no constraint on the distribution, it is impossible to say anything about the samples not seen in the training set \citep{wolpert_lack_1996}. There are various ways to define such constraints. These constraints are usually expressed as an assumption on the simplicity of the target function. %
For example, the Vapnic-Chervonenkis (VC) theory  \citep{VapnikUniformConvergenceRelative1971} defines the concept of the VC dimension which assigns a value %
to each hypotheses set stating how simple/hard it is to learn a hypothesis from that set (learning a harder problems requires more training samples). To use this theory, one needs to assume a hypothesis set with a small VC dimension to learn when a limited number of samples is given.

Another approach for defining and precising the simplicity constraint is to use the information theory. Loosely speaking, in information theoretic terms, an entity is considered simple if it can be described by a few number of bits. One of the main concepts in the information theory is that the expected number of bits required to transmit/save a discrete random variable $X$ with the probability mass function $p(x)$ is controlled by its entropy $H(X)=\EX[-\log p(x)]$ \citep{CoverElementsinformationtheory2012}. Similarly, the mutual information $I(X;Y)=\EX[\log \frac{p(x,y)}{p(x)p(y)}]$ between the two random variables $X$ and $Y$ with the joint distribution $p(x,y)$ can be interpreted as the %
number of bits that each of these random variables contains about the other. 

\label{sec:section_where_IB_is_introduced}
The Information Bottleneck (IB) is a generalization of the sufficient statistics technique proposed in \cite{SlonimDocumentclusteringusing2000} for extracting relevant information from the input random variable $X$ about the target random variable $Y$. It is formulated as %
\begin{equation}
\min_{p(\hat{x}|x)} I(X;\hat{X}) -\beta I(Y;\hat{X})
\end{equation}
where $\hat{x}$ is the extracted feature. %
In a machine learning problem, when using IB as a feature extractor, there are two components, i) encoding the input $X$ to the compressed feature $\hat{X}$ and ii) decoding $\hat{X}$ to find an estimate $\hat{Y}$ of the target variable $Y$. IB has been used in various applications; e.g., time series prediction \citep{XuTimeSeriesPrediction2018}, clustering %
\citep{StrouseInformationBottleneckGeometric2019}, image reconstruction \citep{GhimireImprovingGeneralizationDeep2019}, reinforcement learning %
\citep{GoyalInfoBotTransferExploration2019}, and  distributed learning \citep{FarajiparvarInformationBottleneckMethods2019}.

IB is also employed in the learning theory to harness the mutual information as a measure of simplicity \citep{ShamirLearninggeneralizationinformation2008}. The main underlying intuition is that IB improves the generalization by using $\hat{X}$ which is a compressed (simpler) version of $X$, i.e. it groups the values of $X$ in clusters while trying to preserve the information about $Y$ in each cluster (more precisly, in the general case, where the mapping from $x$ to $\hat{x}$ is stochastic, it will be a soft clustering). Consequently, the problem is reduced to %
estimating the target variable in each cluster which is more robust since each cluster has more data  \citep{SlonimDocumentclusteringusing2000}. 

Recently, this idea has been utilized to explain the reason behind the extraordinary effectiveness of deep neural networks which is considered to be hard to explain %
by using traditional learning theory concepts \citep{ZhangUnderstandingdeeplearning2017}. 
The work in \cite{TishbyDeepLearningInformation2015} put forward the application of IB as a tool to understand deep neural networks %
utilizing information plane diagrams. Then, authors found  empirical evidence %
suggesting that deep neural networks are implicitly solving the IB optimization, i.e. compressing the input while retaining the information about the output \citep{Shwartz-ZivOpeningBlackBox2017}. This idea has been further %
examined 
by %
researchers to understand deep neural networks \citep{KhadiviFlowinformationfeedforward2016,ChengEvaluatingCapabilityDeep2018,ChengUtilizingInformationBottleneck2019,LiuImplementationverificationInformation2018,YuUnderstandingConvolutionalNeural2018} and to improve them \citep{AlemiDeepVariationalInformation2017,NguyenLayerwiselearningstochastic2017,WangDeepMultiviewInformation2019}.

Despite these positive results, the role of IB theory to explain the effectiveness of DNNs is  %
now under debate.
The work in \cite{SaxeInformationBottleneckTheory2018} practically challenged the idea that the information is lost in common DNNs which use ReLU activation functions and claimed that the information loss reported by \cite{Shwartz-ZivOpeningBlackBox2017} might be  due to applying double saturated activation functions (e.g. hyperbolic tangent) and the binning process to estimate the mutual information. 
In another work, \cite{AmjadHowNotTrain2018} argued that the IB optimization is ill-posed for usual scenarios in which the input space is continuous and stated that the previous improvements caused by the application of IB to train DNNs (e.g., \citep{AlemiDeepVariationalInformation2017,KolchinskyNonlinearInformationBottleneck2017}) are due to the fact that they use variational replacement for IB. %
They argued "representational simplicity" and "robustness" are achieved as a positive side effect of this replacement and also some other techniques used along the way. 
More recently, it was theoretically proven that in a common deterministic neural network with any strictly monotone activation function, mutual information is either infinite (continuous input) or constant (discrete input) thus refuting %
applicability of IB to explain %
the generalization power of these models  \cite{GoldfeldEstimatingInformationFlow2019}. 

Various suggestions has been made to remedy these problems, such as explicit noise addition to make the neural network stochastic and lossy \citep{GoldfeldEstimatingInformationFlow2019}, and  replacing the cost function with a related but more well behaved one \citep{KolchinskyCaveatsinformationbottleneck2019} (please see \cite{AmjadLearningrepresentationsneural2019} for a list of these remedies).

In these works, the focus is usually on how to change the method in a way that the promised compression actually happens, but the main idea of why it is desired to compress the input in the first place is less studied. %

Specifically, we found %
that in spite of current hot debates on applicability of IB in %
reasoning about the generalization of DNNs, theoretical studies of this method in the learning theory setting are scarce. In IB literature, it is %
generally considered that "compressed representations generalize better" \citep{WuLearnabilityInformationBottleneck2019}. 
The validity of this proposition %
might be attributed to the results of \cite{ShamirLearninggeneralizationinformation2010} (and the more recent work of \cite{VeraRoleInformationBottleneck2018}) which discuss the generalization properties of IB. However, %
those results are not used explicitly in the discussions surrounding the applicability of IB to analyze the generalization of DNNs. Indeed it was recently demonstrated by \cite{RodriguezGalvezInformationBottleneckConnections2019} that those results are not usable in many scenarios where DNNs are used (such as image classification) since the constants appearing in the bounds are too large (see Section~\ref{sec:section_including_the_previous_bounds} for a full discussion of these results and their implications%
).

It seems that the actual factors promoting the IB as an appealing theory are these two intuitive ideas
i) Occam's razor which states that simpler models generalize better and ii) the concise and ubiquitous %
terminology of information theory in defining simplicity as the number of bits needed to encode a random variable.
However, the relation between these factors and  the generalization was not precisely quantified.%

Recently, Tishby put forward the idea that reducing $H(x)$ can exponentially decrease the number of required samples for achieving a certain accuracy \citep{TishbyStanfordSeminarInformation2018}. 
This exponential decrease in error has the potential to make this theory the first theory of IB usable in practice. But to the best of our knowledge, the idea was never proved theoretically and remained in the intuition arena. %

In this paper, some steps are taken to reduce this theoretical gap. First of all, this intuitive idea that having simpler features (in the information-theoretic sense) 
is beneficial for learning is studied and a tight generalization bound is derived. 
The bound shows exponential dependence on the number of bits when the features are not learned using the dataset. This partly confirms the previous idea, but the analysis shows that there is also a non-eliminable exponential dependence on the desired error. %
Moreover, we show %
that such bounds do not exist when the features are learned using the dataset (which is always the case in DNNs where all layers are trained using back propagation). In other words, even though having low entropy features is beneficial for generalization, when these features are learned from training data, the corresponding generalization bounds can't be used. This also conforms with the previous theoretical results and explain the  intrinsic  cause behind the inapplicability of those bounds in practice. %
We show %
that this is not a side effect of the used theoretical techniques and indeed when the %
features are learned, the limited entropy of the learned features %
can not put enough constraints to guarantee generalization anymore. Thus, %
we argue that this problem can not be eliminated completely unless some extra constraints are incorporated into IB.  %

In the remaining parts of this paper, the main contributions are first summarized in Section~\ref{sec:summary_of_results}.
The results are presented in the next four sections. %
First, the definitions are discussed in Section~\ref{sec:definitions}. 
Then, in Section~\ref{sec:sample_complexity_bounds}, the sample complexity bound $\frac{2^{\slfrac{6H(X)}{\epsilon}}+\log{\frac{1}{\delta}}}{\epsilon^2}$ for the case %
in which the features are fixed is studied. %
After that, in Section~\ref{sec:tightness}, it is shown that the $\frac{1}{\epsilon}$ can not be eliminated and thus there are learning problems with small entropy which are impossible to be learned by limited number of samples available in practice.
Next, in Section~\ref{sec:comparison_with_previous_results} the case where the features are learned using the dataset is studied and it is shown that there is no exponential benefit in this case. %
Finally, the implications of these results are discussed in Section~\ref{sec:discussion}.

\section{Summary of Contributions}
\label{sec:summary_of_results}
Contributions of this paper are three-fold:
\begin{enumerate}
\item \textbf{\small Sample complexity bound for entropy limited distributions}
The bound of $\frac{2^{\frac{6H(X)}{\epsilon}}+\log{\frac{1}{\delta}}}{\epsilon^2}$ is presented for the sample complexity of entropy limited distributions, where %
$H(X)$ is the entropy of the input random variable and $\epsilon$ is the desired error (see Theorem~\ref{theorem:MainTheorem}). This bound shows that there is an exponential dependence on $H(X)$. 
This is a proof for the idea
presented
in \cite{TishbyStanfordSeminarInformation2018} that %
suggests reducing $H(x)$ can exponentially decrease the number of required samples for achieving a certain accuracy. %
A surprising observation made possible by this rigor analysis is the exponential dependence on the term $\frac{1}{\epsilon}$ which, to the best of our knowledge, has not been mentioned in the literature before. 

\item {\bfseries \small Showing existence of hard entropy limited learning Problems}
Indeed %
 we prove that 
there are some distributions which almost meet the aforementioned bound. More precisely, we show that there exist heavy tailed distributions with entropy $H(X)$ which are impossible to be learnt with error $\epsilon$, if the number of samples is less than $2^{\frac{H(X)-1}{\epsilon}}$.
The important implication of this result is that the exponential dependence on the term $\frac{1}{\epsilon}$ is %
intrinsic to the problem and is not a consequence of our approach. This dependence can drastically influence the number of required samples. To see that, consider the case where one wants to achieve the error of at most $\epsilon=0.01$. In this case, even if $H(X)=2$, the number of required samples can be as high as $2^{100}$, which is not reachable in practice. 
This result shows that even if the entropy of the input distribution is small, in order to guarantee arbitrary small error, one might need huge number of samples.
\item {\bfseries \small Separate study of encoder/decoder generalization } 
Our analysis is based on separate study of feature extraction and label estimation from extracted features. In the IB literature, these two subprocesses are called "encoding" and "decoding" respectively, where features are first computed using the encoder and then the label is estimated by the decoder (i.e. it's similar to a classic rate-distortion problem, but the distortion measure is defined as the error of reconstructing the target variable $Y$). 
It is shown that the bound is just valid if the encoder is independent of the training set. If the encoder is learned, the constraint on having a low entropy feature is no longer enough to guarantee generalization. This shows that IB theory is not enough to explain the generalization of DNNs in which the features are also learned. 
\end{enumerate}

\section{Definitions}\label{sec:definitions}
Consider the discrete input space $\mathcal{X}$, the output space  $\mathcal{Y}=\{0,1\}$, an unknown distribution $\rho(x,y)=\rho(x)\rho(y|x)$ on $\mathcal{X}\times\mathcal{Y}$, and a given training set $S=\{(x_i,y_i)\}_{i=1}^N$ %
containing
i.i.d. samples from $\rho(x,y)$, where $N$ is the number of training samples. Suppose the entropy of $\rho(X)$ is $H(X)<\infty$. The goal is to learn a function $f_S:\mathcal{X}\to \mathcal{Y}$ from the training data which minimizes the true zero-one cost $R(f_S)=\EX_\rho[\Indicator(f_S(x)\ne y)]$. However, the training algorithm has only access to the empirical error $R_S(f_S)=\sum_{i=1}^N\Indicator(f_S(x_i)\ne y_i)$ and minimizes it. The generalization gap shows the difference between the empirical error and the true error. It is said that an algorithms achieves $\epsilon$-$\delta$ generalization if 
\begin{equation}
\label{eq:prob_error_epsillon_delta}
\Pr_S\left(\;\left|R\left(f_S\right)-R_S\left(f_S\right)\right|\ge \epsilon\right) \le \delta.
\end{equation}
This characterization of the learning problem is called Probably Approximately Correct (PAC) learning. One of the goals in this analysis is finding the sample complexity bounds for a family $\mathcal{M}$ of distributions. A sample complexity of a learning algorithm is a function  $c_\mathcal{M}(\epsilon,\delta)$, if for all $\rho \in \mathcal{M}$,
$N\ge c_\mathcal{M}(\epsilon,\delta) \Rightarrow \Pr_S\left(\;\left|R\left(f_S\right)-R_S\left(f_S\right)\right|\ge \epsilon\right) \le \delta$.

\section{Sample Complexity Bound}\label{sec:sample_complexity_bounds}
In this section, 
the sample complexity bound for entropy-constrained learning problems is presented. 
In Section \ref{sec:tightness}, the tightness of the achieved bound is studied and it is shown that there are distributions which almost meet this bound.

First, let's sketch the main steps needed to achieve the sample complexity bound of 
\begin{equation}\label{eq:sample_complexity}
c_\mathcal{M}(\epsilon,\delta)= \frac{2^{\frac{6H(X)}{\epsilon}}+\log{\frac{1}{\delta}}}{\epsilon^2}
\end{equation}
where $\mathcal{M}$ is the set of discrete distributions with the entropy of at most $H(X)$.

Comparing Eq. (\ref{eq:sample_complexity}) with the well-known sample complexity of learning with finite hypothesis class $\mathcal{H}$, which is $\frac{\log|\mathcal{H}| +\log{\frac{1}{\delta}}}{\epsilon^2}$, suggests that in some specific sense $2^{2^{\frac{6H(X)}{\epsilon}}}$ can be the effective number of hypotheses when working with entropy-limited distributions. 
Following this observation, the core idea of the proof is to show that the set of all hypotheses $f:\mathcal{X}\to\{0,1\}$ can be partitioned to $2^{2^{\frac{6H(X)}{\epsilon}}}$ subsets %
 such that the probability of the "bad" event $E_{\epsilon}=\Indicator(\left|R_S(h_S)-R(h_S)\right|\ge \epsilon)$ is almost the same for all of the elements in each partition. Then, we assign a center to each partition and replace the learning algorithm by another one that outputs the center of the partition to which the learned hypothesis belongs. Since this new learning algorithm works with a finite size hypothesis space, it is guaranteed to have generalization bounds relative to its size. There are two main difficulties in this approach, i) %
finding the right partitioning function and ii) bounding the error of the original algorithm by the error of the surrogate algorithm working with partition centers.

This partitioning technique (which will be made precise %
later), is related to the covering of hypothesis spaces which is a widely used concept in the statistical learning theory \citep{shalev-shwartz_understanding_2014}. 
On the other hand, in information theory various tools are specifically developed
for effectively counting the probable possibilities. In particular, the sphere covering is used in counting the number of required bits in rate-distortion problem \citep{CoverElementsinformationtheory2012}. 
Based on intuitions from these information theoretic results, there is a heuristic stating that the effective number of hypotheses should be about $2^{2^{H(X)}}$ \citep{TishbyStanfordSeminarInformation2018}. 
The reason is the loosely stated theorem that indicates a random variable with entropy $H(X)$ can be represented by $H(X)$ bits. Therefore, effectively, there are $2^{H(X)}$ values which one needs to consider instead of $X$s. %
Finally, the number of all functions on those values will be $2^{2^{H(X)}}$. Unfortunately, this simple intuitive reasoning %
differs from the true answer by an important $\frac{1}{\epsilon}$ factor on the exponent of the exponent! (the root of this mistake is explored thoroughly in Remark~1 at the end of this section and in Appendix~\ref{sec:appendix_factorized_distribution}) %
In the remaining parts of this section, the more elaborate approach to analyze this problem is described. After that the root of the fallacy in previous arguments is discussed (see Remark~1).

To begin the discussion, some lemmas are presented. First, it is proven that it is possible to bound the probability of $\epsilon$ error (Eq. (\ref{eq:prob_error_epsillon_delta})) by replacing the set of all functions $\mathcal{F}$ with a finite set $\mathcal{F}_g\in\mathcal{F}$. This is done by using a partitioning functional $g:\mathcal{F}\to\mathcal{F}_g$ which estimates every function in $\mathcal{F}$ with a function in $\mathcal{F}_g$ (see Fig.~\ref{fig:NNMC}). Lemma \ref{lemma:3part_error_decomposition} shows that the error can be controlled by, i) the size of $\mathcal{F}_g$ and ii) the error introduced by replacing $f$s with $g(f)$s.

\begin{lemma}[Three Parts Error Decomposition]
\label{lemma:3part_error_decomposition}
Let $g:\mathcal{F}\to\mathcal{F}_g$ be an arbitrary partitioning functional which maps $\mathcal{F}$ to the finite sized set $\mathcal{F}_g\subset \mathcal{F}$ (the set of \emph{centers}), and $h_S=\mathcal{A}(S)$ be the learned hypothesis by the learning algorithm. The generalization gap can be decomposed as
\begin{align}
\label{eq:3part_error_decomposition}
\Pr\{ | R_S(h_S)-&R(h_s) | \ge 3\epsilon \}
\\
\le \;\;\Pr \{& \;\; |R_S(g(h_S))-R(g(h_S))|\ge\epsilon \;\;\}\nonumber\\
+\Pr \{& \;\; |R(g(h_S))-R(h_S) |\;\;\;\;\; \ge\epsilon \;\;\}\nonumber\\
+ \Pr \{& \;\; |R_S(g(h_S))-R_S(h_S)| \;\;\ge\epsilon \;\; \}\nonumber.
\end{align}

\end{lemma}

\begin{proof}
\begin{align*} 
\Pr\{ | R_S(h_S)-&R(h_s) | \ge 3\epsilon \}\\
=\Pr \{ \;\; |&R_S(h_S)-R_S(g(h_S)) \\
+ &R_S(g(h_S))-R(g(h_S))\\
+ &R(g(h_S))-R(h_S) |  \;\; \ge 3\epsilon \;\;\}\\
\le \Pr \{ \;\; &|R_S(h_S)-R_S(g(h_S))| \\
+ &|R_S(g(h_S))-R(g(h_S))|\\
+ &|R(g(h_S))-R(h_S) |  \;\; \ge 3\epsilon \;\;\}\\
\le \Pr \{ \;\; &|R_S(h_S)-R_S(g(h_S))| \;\;\ge\epsilon \\
\cup \; &|R_S(g(h_S))-R(g(h_S))|\ge\epsilon\\
\cup \; &|R(g(h_S))-R(h_S) |\;\;\;\;\; \ge\epsilon \;\;\}\\
\le \;\;\Pr \{& \;\; |R_S(h_S)-R_S(g(h_S))| \;\;\ge\epsilon \;\; \}\\
+ \Pr \{& \;\; |R_S(g(h_S))-R(g(h_S))|\ge\epsilon \;\;\}\\
+\Pr \{& \;\; |R(g(h_S))-R(h_S) |\;\;\;\;\; \ge\epsilon \;\;\}\\
\end{align*}
\end{proof}

\begin{figure}
\centering
\includegraphics[width=0.7\linewidth]{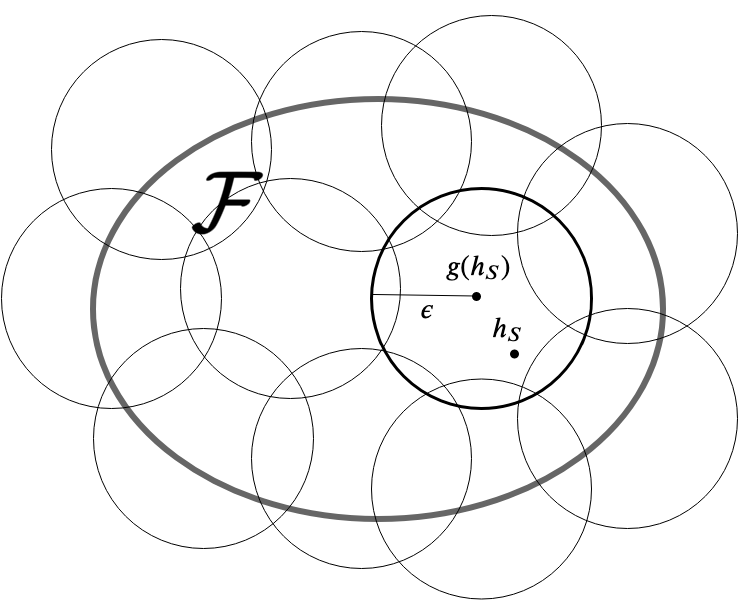}
\caption{Covering the space of functions. Schematic view of the space of functions $\mathcal{F}:\mathcal{X}\to\mathcal{Y}$ augmented with a metric based on the difference of error of hypotheses.  The large space $\mathcal{F}$ is covered with a finite set of balls of radius $\epsilon$. A learned hypothesis $h_S$ is mapped to the nearest center $g(h_S)$. For a target error of $\epsilon$, the effective number of hypotheses is related to the number of balls (as is made precise in Lemma~\ref{lemma:3part_error_decomposition}).}
\label{fig:NNMC}
\end{figure}

It should be emphasized that this lemma provides a general tool to study generalization.
The strength of this result comes from the fact that the partition function can depend on the distribution 
of data %
(though it cannot depend on the training set). 
Thus, in order to analyze a
generalization gap for a family of distributions (in this case the entropy-limited distributions),
one can try to see if it is possible to suitably partition the function space for each of the distributions in the family.

Note that the first term in the decomposition is the generalization gap of the learning algorithm $\mathcal{A}_g(S)=g(h_S)$. Since $\mathcal{A}_g(S)\in \mathcal{F}_g$, by using the classic generalization theorem for finite hypotheses sets, this term is bounded by $\left|\mathcal{F}_g\right| 2e^{-2N \epsilon^2}$. Thus, the first term is controlled by the number of centers. The next two terms are controlled by the effectiveness of the partitioning functional to assign a function $h_S$ to a center $g(h_S)$. More precisely, a good partitioning functional will assign a center to $h_S$ which has a similar error to it on both true distributions (second term) and the empirical distribution (third term). %

The remaining challenge is to find a good partition of the function space which can control all three terms. In order to do that, it is first shown that the second and third terms can be further bounded by a suitable distance measure in the space of functions. This makes it possible to reformulate the problem as finding a covering for the space $\mathcal{F}$ augmented with that distance. This is done in Lemma~\ref{lemma:bounding_error_by_probability_on_rho_x}, where it is shown that there is a general relation between %
zero-one loss differences of two functions and their (weighted) hamming distance.

\begin{lemma}
\label{lemma:bounding_error_by_probability_on_rho_x}
For every joint distribution $\rho(x,y)=\rho(x)\rho(y|x)$ and every pair of functions $f,f^\prime \in \mathcal{F}$, where $\mathcal{F}$ is the set of all binary classifiers $f:\mathcal{X}\to \{0,1\}$, we have
\begin{equation}\label{eq:bounding_error_by_probability_on_rho_x}
\left|R_\rho(f) -R_\rho(f^\prime)  \right| \le \Pr\nolimits_{\rho(x)}\{f(x)\ne f^\prime(x)\}
\end{equation}
where $R_\rho(f)=\Pr\nolimits_{\rho(x,y)}\{f(x)\ne y\}$.
\end{lemma}

\begin{proof}
\begin{align*}
|R_\rho(f) -R_\rho(f^\prime)  |&=\left|\EX_{x,y}\left[\Indicator( f(x)\ne y) - \Indicator( f^\prime(x)\ne y)\right]\right|\\
&=\left|\Pr\nolimits_{\rho(x)}\{f(x)\ne y \;\text{and}\; f^\prime(x)= y\}\right.\\
&\;\;\;-\left.\Pr\nolimits_{\rho(x)}\{f(x)= y \;\text{and}\; f^\prime(x)\ne y\}\right|\\
&\le \Pr\nolimits_{\rho(x)}\{f(x)\ne f^\prime(x)\}
\end{align*}
\end{proof}

Note that Lemma~\ref{lemma:bounding_error_by_probability_on_rho_x} has an important side benefit: it removed the dependence of the terms on $\rho(y|x)$ (note that the l.h.s of Eq. (\ref{eq:bounding_error_by_probability_on_rho_x}) depends on both $\rho(x)$ and $\rho(y|x)$ while the r.h.s depends only on  $\rho(x)$). This property becomes critical later since the bound can only depend on $\rho(x)$ (more precisely, it should only depend on the entropy $H(X)$). 

Now, we are ready to continue the quest of finding a suitable partitioning function. In order to do that, the main idea is to focus on the second term of the inequality (\ref{eq:3part_error_decomposition}) and find a partitioning which zeros out the second term while using a small number of centers. Later, it will be proven that the third term (which is an empirical counterpart %
of the second term) will also be controlled. Note that using Lemma~\ref{lemma:bounding_error_by_probability_on_rho_x} on the second term, the problem can be solved by finding an $\epsilon$-covering for the space $\mathcal{F}$ augmented with the distance function $d(f,f^\prime)=\Pr\nolimits_{\rho(x)}\{f(x)\ne f_g(x)\}$.

Therefore, the first two terms of Eq.(\ref{eq:3part_error_decomposition}) are controlled if  it is possible to find an $\epsilon$-covering  with a small size. %
It is possible as stated in Lemma 3.
	
\begin{lemma}%
\label{lemma:covering_entropy_limited_functions}
For every distribution $\rho(x)$ with finite entropy $H(X)$, for all $\epsilon\in(0,1)$, there is a set $\mathcal{F}_\alpha\subseteq \mathcal{F}$ with finite size $\left|\mathcal{F}_\alpha\right|\le 2^{2^{\frac{H(X)}{\epsilon}}}$, where $\forall f\in \mathcal{F};\; \exists f^\prime \in \mathcal{F}_\alpha;\; \Pr\nolimits_{\rho(x)}\{f(x)\ne f^\prime(x)\}\le \epsilon$. 
\end{lemma}

Before proving this lemma,  it is needed to define a suitable partitioning functional called \emph{$\alpha$-probable input projection}.

\begin{mathdef}[$\alpha$-probable input projection]\label{definition:alpha_partition_function}
Given a distribution $\rho(x)$ on the discrete set $\mathcal{X}$ and a function $f:\mathcal{X}\to\{0,1\}$, the $\alpha$-probable input projection is the functional $g_\alpha:\mathcal{F}\to\mathcal{F}_\alpha$ defined as
\begin{align}
\left(g_\alpha(f)\right)(x)=\left\{
\begin{aligned}
&f(x) \;\;\;& x\in\mathcal{X}_{\ge \alpha}\\
&0   \;\;\; & o.w.
\end{aligned}
\right.
\end{align}
where $\mathcal{X}_{\ge\alpha}=\set{x|\rho(x)\ge\alpha}$ is the set of high probable input values and $\mathcal{F}_\alpha$ is the range of $g_\alpha$ which is the set of all functions which are zero on $x\notin\mathcal{X}_{\ge\alpha}$.
\end{mathdef}

The partitioning functional $g_\alpha$ has a simple logic: just consider the value of function on the most probable elements of $\mathcal{X}$ and ignore all others. Therefore, using this functional, two functions are in the same partition iff they are the same on $\mathcal{X}_{\ge\alpha}$. %

It will be shown that by choosing the correct $\alpha$, the functional $g_\alpha$ is the suitable partitioning. 
It is done by first showing that $g_\alpha$ can be used to find an $\epsilon$-covering with the size less than $2^{2^{\frac{H(X)}{\epsilon}}}$ for the space $\mathcal{F}$ augmented with the distance function $d(f,f^\prime)=\Pr\nolimits_{\rho(x)}\{f(x)\ne f_g(x)\}$
 (Lemma~\ref{lemma:covering_entropy_limited_functions}), then bounding the error of such covering on the empirical distribution (Lemma~\ref{lemma:covering_error_on_empirical_distribution}).

\begin{proof}[Proof of Lemma~\ref{lemma:covering_entropy_limited_functions}]
It will be shown that using $f^\prime=g_\alpha(f)$ as defined for $\alpha$-probable input projection (Definition~\ref{definition:alpha_partition_function}) with $\alpha=2^{-\frac{H(X)}{\epsilon}}$ is enough to achieve this result.

Note that by definition $\forall x \in \mathcal{X}_\ge\alpha \Rightarrow f(x)=f^\prime(x)$. Therefore 
\begin{align*}
 \Pr\nolimits_{\rho(x)}\{f(x)\ne f^\prime(x)\}\le  \Pr\nolimits_{\rho(x)}\{x\notin \mathcal{X}_{\ge\alpha}\}.
\end{align*}
This probability can be bounded by using the Markov inequality
\begin{align}
\Pr\nolimits_{\rho(x)}\{x\notin \mathcal{X}_{\ge\alpha}\}
=&\Pr\nolimits_{\rho(x)}\{\rho(x) < \alpha\}\nonumber\\
\label{eq:deviation_bound_for_entropy}
=&\Pr\nolimits_{\rho(x)}\{\log \frac{1}{\rho(x)} > \log\frac{1}{\alpha}\}\\
\le&\frac{H(X)}{\log \frac{1}{\alpha}} =\epsilon.\nonumber
\end{align}

On the other hand, $\mathcal{F}_\alpha$ is the set of functions which are zero outside $\mathcal{X}_{\ge\alpha}$. Therefore, $\left|\mathcal{F}_\alpha\right|=\left|2^{\mathcal{X}_{\ge\alpha}}\right|$. Since every element of  $\mathcal{X}_{\ge\alpha}$ has a probability of at least $\alpha$, it is guaranteed that $\left|\mathcal{X}_{\ge\alpha}\right|\le \frac{1}{\alpha}$. As such, it is concluded that  $\left|2^{\mathcal{X}_{\ge\alpha}}\right|\le 2^\frac{1}{\alpha}=2^{2^{\frac{H(X)}{\epsilon}}}$.

\end{proof}

\begin{lemma}\label{lemma:covering_error_on_empirical_distribution}
For every distribution $\rho(x)$ and set of samples $S=\{x_i\}_{i=1}^N\distas{i.i.d}\rho(x)$
\begin{equation}
\Pr_S \{\;\; |R_S(h_S)-R_S(g_\alpha(h_S))| \;\;\ge\epsilon \;\; \}\le \exp(-2N(1-r)^2\epsilon^2)
\end{equation}
where $0 \le r \le 1$ is a constant, 
$\alpha=2^{-\frac{H(X)}{r\epsilon}}$ and 
the partitioning function $g_\alpha$ is defined in Definition~\ref{definition:alpha_partition_function}.
\end{lemma}
\begin{proof}
Using Lemma~\ref{lemma:bounding_error_by_probability_on_rho_x} for the empirical distribution $\hat{\rho}_{_S}$, we have
\begin{equation}
\Pr_S \{\; |R_S(h_S)-R_S(g(h_S))| \;\;\ge\epsilon \; \}\le \Pr_S\{\hat{\rho}_{_S}(x\notin \mathcal{X}_{\ge\alpha})\ge \epsilon\}.\nonumber
\end{equation}
Define $\theta_\alpha=\rho(x\notin \mathcal{X}_{\ge\alpha})$, this probability can be bounded using the Hoeffding's inequality as
\begin{eqnarray}
\Pr_S\{\hat{\rho}_{_S}(x\notin \mathcal{X}_{\ge\alpha})\ge  \epsilon\}\le \exp(-2N(\epsilon-\theta_\alpha)^2).
\end{eqnarray}
Now, by employing %
the Markov inequality to bound $\theta_\alpha$ as in Lemma \ref{lemma:covering_entropy_limited_functions}, it is observed that $\theta_\alpha\le r \epsilon$ which proves the result.
\end{proof}

Previous two lemmas show that the $\alpha$-probable input projection partition has a small approximation error on both true and empirical distributions. Now all that is remained is to use Lemma~\ref{lemma:3part_error_decomposition} to find a bound on the sample complexity of the entropy limited distributions.
This is done in Theorem~\ref{theorem:MainTheorem}.

\begin{theorem}
[Sample Complexity of Learning Entropy Limited Distributions]
\label{theorem:MainTheorem}
For every discrete distribution $\rho(x,y)=\rho(x)\rho(y|x)$, where the input distribution has the entropy $H(X)$, for every learning algorithm $\mathcal{A}$
 and for all $\epsilon>0$ and $\delta>0$,  if 
$N \ge \frac{2^{\frac{6H(X)}{\epsilon}}+\log(\frac{1}{\delta})}{\epsilon^2}$, 
we have the $\epsilon$-$\delta$ generalization; i.e., $\Pr\left(\;\left|R\left(\mathcal{A}(S)\right)-R_S\left(\mathcal{A}(S)\right)\right|\ge \epsilon\right) \le \delta$.

\end{theorem}

\begin{proof}

Using the error decomposition technique of  Eq. (\refeq{eq:3part_error_decomposition}) we have

\begin{align*}
\Pr\{ | R_S(h_S)-&R(h_s) | \ge \epsilon \}
\\
\le \;\;\Pr \{& \;\; |R_S(g(h_S))-R(g(h_S))|\ge\epsilon^\prime \;\;\}\nonumber\\
+\Pr \{& \;\; |R(g(h_S))-R(h_S) |\;\;\;\;\; \ge\epsilon^\prime \;\;\}\nonumber\\
+ \Pr \{& \;\; |R_S(h_S)-R_S(g(h_S))| \;\;\ge\epsilon^\prime \;\; \}.\nonumber
\\
\le\;\;\;\;\;\;(&2^{2^\frac{H(X)}{r \epsilon^{\prime}}})2\exp(-2N\epsilon^{\prime^2})\nonumber
\\
+ \;\;\;\;\;&0\nonumber
\\
+\;\;\;\;& \exp(-2N(1-r)^2\epsilon^{\prime^2}).\nonumber
\end{align*}
where $\epsilon^\prime=\frac{\epsilon}{3}$ and the partitioning functional $g_\alpha$ of definition~\ref{definition:alpha_partition_function} is utilized with $\alpha=2^{-\frac{H(X)}{r\epsilon}}$. 
In the second inequality each of the three terms is bounded respectively. The first term is the generalization gap of a learning algorithm with finite hypotheses set size, the second term is zero as proved in Lemma \ref{lemma:covering_entropy_limited_functions} where $g_\alpha$ covers the space of functions, and the third term is bounded using Lemma \ref{lemma:covering_error_on_empirical_distribution}. Finally, solving for $n$ which guarantees $\epsilon-\delta$ error, we get
\begin{equation*}
N\ge \frac{2^{\frac{H(X)}{r\epsilon^{\prime^2}}}+2+\log\frac{1}{\delta}}{2\log \text{e} (1-r)^2\epsilon^{\prime}}
\Rightarrow \Pr\{ | R_S(h_S)-R(h_s) | \ge \epsilon \}\le\delta
\end{equation*}
where using $r=\frac{1}{2}$ concludes the proof.

\end{proof}

\subsection*{ \textbf{Remark 1:} A Discussion about Typicality of Each Input Variable}
\label{sec:problem_with_AEP_arguments}
As mentioned %
earlier, the previous conjecture of \cite{TishbyStanfordSeminarInformation2018} lacked the $\frac{1}{\epsilon}$ factor in the exponent. One of the arguments therein is based on the assumption that the distribution of random variable $X$ is factorized to many components (which is a reasonable assumption in tasks like image processing where each image is composed of many local patches). Using this assumption it is argued that  for a large enough input, we can use AEP (Asymptotic Equipartition Property) to show that with probability 1 the input variable $X$ will be "typical" and $\rho(X)=2^{H(X)}$. Using this argument it might seem that the $\frac{1}{\epsilon}$ factor
can be eliminated. Here it is demonstrated that it's not the case and the reason AEP can not be used to found the generalization gap. 
To do that first the AEP and the "loose" interpretations of it are presented.%

AEP states that for a finite valued stationary ergodic process $X=\{W_\ell\}$, with probability $1$ we have
\begin{equation}
\label{eq:AEP}
\lim_{L\to\infty}-\frac{1}{L} \log p(W_1,\dots W_L)=H_w,
\end{equation}
where $H_w$ is the entropy rate of the process\citep{CoverElementsinformationtheory2012}. 

Loosely speaking, this theorem states that as $L$ gets larger, the joint probability of every probable sequence gets the same and is $2^{-LH_w}$. This observation is used to define the key concept of "typical set" as the set of these  equiprobable sequences which its size grows as $2^{LH_w}$. This definition is one of the key concepts in information theory which simplifies the matters a lot since one usually can just focus on sequences inside this equiprobable set and ignore others.

Notice that in the machine learning setting of our interest, the whole sequence $X=(W_1,\dots W_L)$ is a given input and thus $H(X)=LH_w$. Now if one could show that with probability $1$ the input is from a set of size $2^{H(X)}=2^{LH_w}$, the problem would be easily solved and number of functions would be $2^{2^{H(X)}}$ without any $\slfrac{1}{\epsilon}$ factor. 
This is the basic intuition behind the previous conjecture: since in applications like image processing, $L$ is quite large, we can assume that there is actually a ("typical") set of size $2^{H(X)}$ with probability $1$. 

Unfortunately this loosely stated argument does not work. To see the problem first of all note that AEP is an asymptotic phenomena for $L\to\infty$, but in that case $H(X)=LH_w\to\infty$ as well, which makes the generalization bound meaningless. Moreover in machine learning setting it is assumed that there is a fixed distribution for generating each sample. In other words each sample has a fixed size with a limited $H(X)$ and we are not dealing with an overgrowing stream of input.
Thus it is clear that such reasoning can not be directly used. Actually a closer look at the AEP in eq~(\ref{eq:AEP}) gives us a hint %
that the exponent of the bound implicitly depends on $\epsilon$.

To understand this point, the $\lim$ in eq~(\ref{eq:AEP}) should be expanded which gets
\begin{equation}
\forall \eta>0; \exists L_0; \forall L>L_0;  |-\frac{1}{L} \log p(W_1,\dots W_L) - H_w| <\eta.
\end{equation}
This equation makes it explicit that the value of $L_0$ depends on the required precision $\eta$. Then as $L>L_0$ and $H(X)=L H_w$, it is evident that the exponent in the generalization bound $2^{H(X)}$ grows as the required precision is getting smaller. Compare this with our analysis in which $H(X)$ is fixed, and the dependence gets explicit by the factor of $\slfrac{1}{\epsilon}$

This discussion summarized the fallacies in the previous argument based on AEP to prove the generalization bound. But it does not show if it is possible to fix these seemingly technical problems by a more precise argument. Unfortunately this is not the case. 
In Appendix~\ref{sec:appendix_factorized_distribution} we dive deeper into analyzing the generalization bound in the presence of this additional assumption that the distribution is factorized to many components. There it is demonstrated more throughly  how AEP fails to improve the bounds. Moreover it is shown that while there are tighter deviation bounds when using this assumption, the final bound does not vary much.

\subsection*{\textbf{Remark 2:} Compressed Input vs Compressed Model}
It is worth mentioning the relation between the discussion in this section and another application of information theory in machine learning introduced in \cite{RussoHowmuchdoes2015,BassilyLearnersthatUse2018}.
These studies are based on using the term $I(S;\mathcal{A}(S))$ to control the generalization. %
This term tries to restrict the amount of information stored in the whole learned hypothesis. We use the term "compressed model" to describe this. It should be noted that %
this viewpoint is different from the "input compression" bounds studied in this article
in which the information retained inside each sample is controlled (even though the term information bottleneck was also sometimes used for the former case; %
e.g., the work in  \cite{AchilleInformationDropoutLearning2018}). One of the main practical differences between these two methods is that $I(S;\mathcal{A}(S))$ highly depends on the algorithm $\mathcal{A}$. Despite the fundamental differences between these approaches, there is a %
relation between %
their obtained results: in that approach the number of hypotheses is bounded by $2^{H(\mathcal{A}(S))}$ (in the deterministic case - see Section 3.1 of \cite{BassilyLearnersthatUse2018}) which represents in a way how many effective hypotheses are used by the algorithm $\mathcal{A}$. In the %
approach discussed here, this term is replaced by $2^{2^{\frac{6H(X)}{\epsilon}}}$ which again was proven to bound the effective number of functions (but note that the other parts of the bound as well as the approach required for the proof are quite different). %

\section{Tightness of Sample Complexity Bound} \label{sec:tightness}
In the previous section, an upper bound for the sample complexity was achieved. Now, the natural question that raises is how tight this bound is. Specially, one might be interested to see if the exponential dependence on $\frac{1}{\epsilon}$ can be removed. In this section, it is shown that despite all the inequalities used to achieve the bound, the final result is almost tight. More precisely, it is shown that there are actually distributions with limited entropy $H(X)$, which cannot be learnt using fewer than $2^{\frac{H(X)-1}{\epsilon}}$. To show this, first the concept of \emph{Heavy Tailed Entropy Limited Distributions}(HTELDs) is introduced and then it is shown that the learning problems involving such distributions can be intrinsically hard regardless of the algorithm used to solve them.

\begin{mathdef}[Heavy Tailed Entropy-Limited Distribution]
A  distribution $\rho(x)$ on a discrete set $\mathcal{X}$ is called a heavy tailed entropy-limited distribution with parameters $(\gamma,\epsilon,\alpha)$, if it has entropy  $H(X)=\gamma$ and there is $\mathcal{X}_{\le\alpha}\subset\mathcal{X}$ with probability $\rho(\mathcal{X}_{\le\alpha})\ge\epsilon$ such that  $\forall x\in\mathcal{X}_{\le\alpha} : \rho(x)\le\alpha$. It is considered that $\alpha\ll\epsilon$ and $M=\left|\mathcal{X}_{\le\alpha}\right|$ is large (consistent with the name heavy tailed).

\end{mathdef}

Note that when sampling from a Heavy Tailed Entropy-Limited Distribution (HTELD) distribution, it is hard to make sure that a large portion of $\mathcal{X}_{\le\alpha}$ is presented in the sample set. The reason is that this set is %
made from a large number of elements with very small probabilities. 

It will be shown that there exist HTELD distributions with $\alpha\lesssim 2^{-\frac{H(X)}{\epsilon}}$. To see this, consider the set $\mathcal{X}=\{0,1,\dots,M\}$ and the distribution  
\begin{equation}
\rho_{\epsilon,\alpha}(x)=\left\{ \begin{aligned}
&1-\epsilon \;\;\;\; &x=0\\
&\alpha \;\;\;\; &\text{o.w.}%
\end{aligned}\right.
\end{equation}
Suppose $\epsilon$ has a value near $0$%
, thus most of the mass is put on the first element and the rest is uniformly distributed among others. The entropy of the distribution $\rho(x)$ is $H(X)=-(1-\epsilon)\log(1-\epsilon)-(M\alpha)\log(\alpha)$. Note that $M=\slfrac{\epsilon}{\alpha}$, solving for $\alpha$, it is achieved that $\alpha=2^{-\frac{H(X)+(1-\epsilon)\log(1-\epsilon)}{\epsilon}}$
and
$M= 2^{\frac{H(X)-H(\epsilon)}{\epsilon}}$
where $H(\epsilon)=-(1-\epsilon)\log(1-\epsilon)-\epsilon\log(\epsilon)$.

To make this example clear, consider $\epsilon=0.01$. Therefore, even for a small value of $H(X)=1$, we have $M\simeq 2^{94}$. Now, notice that in order to reduce the error bellow $\epsilon$, one should enter the regime where the %
labels of samples $x\in\{1,\dots m\}$ are being learned. Since there is a $1-\epsilon$ chance to have a sample from this set and also the distribution on this set is uniform it is a challenging task. %
For example, to achieve error less than $\frac{\epsilon}{2}$, %
if there are no other assumptions on the distribution $\rho(x,y)$, it is required to have at least $n\simeq 2^{93}$ samples to make sure that at least half of the tail elements are observed and therefore the error is in the desired range.

It is worth to explicitly pointing out the reason why the term $\frac{1}{\epsilon}$ was appeared in the formulation despite the usual intuitive idea stating that the number of required bits to compress a signal $X$ is $H(X)$. 
In fact, this intuition can be employed in one of the following scenarios %
(and does not hold in the learning theory setting that our work belongs to)
i) when there are $n\to\infty$ i.i.d. realizations of r.v. $X$ and all of them are compressed together, the "average" number of bits used for each of the samples is $H(X)$, ii) when one wants to send a single sample of $X$, there exists a code (e.g., Huffman coding) with "expected" length of $H(X)+1$ sending the sample without error. Therefore, that intuition would work if either there were large blocks of samples compressed together or the "expected" %
number of bits were the parameter of interest. However, none of these scenarios are the case for the learning theory problem. The former one is not the case since %
the loss for each sample must be calculated independent of the other samples. The later one is not the case since in the analysis it was required to guarantee a "high probability" bound on the error. This required to control the mass in the tail of the distribution and further introduced a $\frac{1}{\epsilon}$ factor to the expected value using Markov inequality.

It should be added that even though the proved bound may produce trivial results in many real world cases, it can still be used to compare two algorithms/problems which differ only in the value of $H(X)$. This is a common argument in the statistical learning literature where sometimes a bound is loose and produces trivial results, but nevertheless it is used to compare algorithms together, if it is "equally loose" for different  %
algorithms (e.g., see Chapter 2 of  \cite{Abu-MostafaLearningdatashort2012} for a discussion about looseness of VC dimension bounds).

\section{Generalization for a Learned Feature Extractor}
\label{sec:comparison_with_previous_results}
As mentioned in Section~\ref{sec:section_where_IB_is_introduced}, IB is a popular approach for feature extraction from the input in a compressive while predictive manner which can estimate the target from features. IB has been employed in many applications and also used as a way to understand and improve the generalization of learning algorithms. 
Nonetheless, there are not many theoretical studies about IB in the learning theory setting and the few available theoretical results are not such impressive. 
In this section, we shed light on the looseness of IB generalization bounds in the existing studies and the reason behind this looseness.

In %
previous sections,  the effect of having a low entropy input distribution on the learnability of a problem was studied. The discussion showed that the generalization gap of learning a distribution which is compressible to $k$ bits has an exponential dependence on $k$. It is in fact consistent with the idea that a simpler target %
should be easier to learn. A valid question would be what one can do if the input distribution has a high or even infinite entropy. 
The solution suggested by IB
is to map the input variable $X$ to a compressed feature space $\hat{X}$ which is simpler to learn. This idea can be expressed by the Markov chain $Y-X-\hat{X}-\hat{Y}$, where $\hat{Y}$ is the estimated label of the model. The process of calculating $\hat{X}$ from $X$ is called encoding and the calculation of $\hat{Y}$ from $\hat{X}$ is called decoding as mentioned before. The IB defines the optimization problem as 
\begin{equation}
\label{eq:generalized_IB_relaxed}
\min_{p(\hat{x}|x)} \mathcal{C}(p)-\beta I(Y; \hat{X})
\end{equation}
where $\mathcal{C}$ is the \emph{compression} criteria %
which maps the joint distribution $p(x,\hat{x})=p(x)p(\hat{x}|x)$ to a positive real value (e.g. using entropy of learned features to define the compression, we can use  $\mathcal{C}(p)=H(\hat{X})$). The second term %
ensures that the information about $Y$ is not lost due to compression and $\beta$ controls the tradeoff between compression and information preservation. %
When there exists a sufficient statistic, it is possible to maintain all the information about $Y$ with a significant compression of $X$ \citep{Tishbyinformationbottleneckmethod2000}. As such, IB can also be viewed as the Lagrange relaxation of the following optimization problem
\begin{equation}
\label{eq:generalized_IB}
\min_{\substack{p(\tilde{x}|x)\\ s.t. I(Y; \hat{X})=I(Y; X)}} \mathcal{C}(p),
\end{equation}
which is shown to have the minimal sufficient statistics as its solutions (when they exists) \citep{ShamirLearninggeneralizationinformation2010}. The optimization problem in (\ref{eq:generalized_IB_relaxed}) has the benefit to be applicable even %
when the sufficient statistics do not exist.

There are two flavors for IB which differ on how the compression criteria $\mathcal{C}$ is defined.
Traditionally, IB used the notion of mutual information $\mathcal{C}(p)=I(X;\hat{X})$ to define compression in information theoretic terms. A more recent line of research is the deterministic IB proposed by  \cite{StrouseDeterministicInformationBottleneck2017} which uses the entropy $\mathcal{C}(p)=H(\hat{X})$ 
as the defined criteria for compressed features.
This deterministic setting is readily matched with our analysis of learning based on %
entropy constrained distributions. 

\subsection{Limitation of IB Generalization}

In order to compute the terms involved in the IB optimization problem, the distribution $p(x)$ is required. The reason is that, IB was originally developed for the cases %
in which the distribution is known. 
Nonetheless, in order to make it applicable in the learning setting where the distribution is unknown, it was suggested to simply plug-in the empirical distribution to solve the problem. This approach was theoretically analyzed in  \cite{ShamirLearninggeneralizationinformation2010} and is the defacto approach to deal with this problem. Unfortunately, %
it causes the dependence on the dataset and has the potential to worsen the generalization bounds since now the distribution $p(\hat{X})$ (which is the compressed distribution used to predict $Y$) depends on the dataset. %
In fact, it's arguable that this is the main reason of having poor generalization bounds for IB. The cause is that unless one puts some restrictions on the family of encoders $p(\hat{x}|x)$, the capacity of encoders with the form $\min_{p(\hat{x}|x)} \ell(S; p(\hat{x}|x))$ is too high to guarantee generalization. This point gets clear in the following discussion.

The analysis of this paper has the benefit of separately studying the generalization of encoder and decoder.
Discussions about generalization of entropy limited distributions in Sec~\ref{sec:sample_complexity_bounds} can be utilized to analyze the decoding part of a system which uses IB (i.e., after the features are learned and when the labels should be decoded from them). However, to fully analyze the performance of IB in which the features are learned from the training set, the overfitting due to learning the encoder should also be considered. We show that
the exponential effect of the low entropy on improving the generalization, is not %
applicable for the trainable encoder case which leads to much looser bounds. %
In other words, if there is a fixed feature map with limited entropy, the value of entropy can control the generalization gap, but if the encoder is learned using the training samples%
, there is no longer an exponential benefit for having limited entropy.
Furthermore, it is shown that this is an intrinsic property of IB and not a byproduct of our employed theoretical tools. It is shown that this result conforms with the previous achieved generalization bounds for IB \citep{ShamirLearninggeneralizationinformation2010, VeraRoleInformationBottleneck2018,RodriguezGalvezInformationBottleneckConnections2019} (see Section~\ref{sec:existing studies}).

\subsection{Illustrative Example}
To start, a simple but important example is presented which will be a base for the further analysis of generalization. This example shows the power of trainable encoder to overfit to training data, despite restricting the entropy of learned features to be low.

\begin{exmp}\label{exmp:overfitting_encoder}
Consider a balanced binary classification problem with a zero Bayes error. Regardless of how complicated the underlying input distribution is, there is always the deterministic feature extractor which overfits to labels of the training set; i.e.
\begin{equation}
\label{eq:overfitting_encoder}
p(\hat{x}|x_i)=\left\{ \begin{aligned}
&1 & \hat{x}=y_i \\
&0 & o.w.
\end{aligned} \right.
\end{equation}
where $(x_i,y_i)$ is %
an input-output pair in the training set. 

Obviously, this representation of $x$ holds the whole information about $y$ on the empirical distribution. It is also a quite compressed representation as $\hat{I}(X;\hat{X})=\hat{H}(\hat{X})=\hat{H}(Y)=1$ on the empirical distribution (using the balanced dataset assumption). Thus, $p(\hat{x}|x)$ is a solution of the optimization problem (\ref{eq:generalized_IB}) for either the choice of $\mathcal\mathcal{C}(p)=H(\hat{X})$ or $\mathcal\mathcal{C}(p)=I(X; \hat{X})$ when the empirical distribution is used. In other words, the encoder $p(\hat{x}|x)$ alone has the power to overfit the dataset, and the mutual information terms cannot interfere to constrain this capacity. To explain it more precisely using VC theory, it can be said that the class of unconstrained conditional distributions $p(\hat{x}|x)$, where $|\hat{\mathcal{X}}|\ge2$ has the ability to shatter every set of distinct samples. Thus the learning algorithm which harness the full extent of this power to fit the training samples, is similar to one working with a hypotheses set with infinite VC dimension, and thus could result in a large generalization gap. %
Furthermore, note that no assumptions on the distribution of data and no restrictions on the learned model $p(\hat{x}|x)$ are made. By "no free lunch" theorem, the test error can be large despite the fact that the training error is always zero (and therefore the generalization gap could be large).%

\end{exmp}

Since there exist such cases in IB solutions set on the empirical distribution, one cannot hope to have strong and general generalization bounds. As a result, the achieved bounds need to have strong assumptions on the input distribution. %

\subsection{Ineffectiveness of Existing Theoretical Studies} 
\label{sec:existing studies}
In this subsection, the previous related studies in the literature are more elaborately investigated. The theoretical studies of IB are mostly done in %
\cite{ShamirLearninggeneralizationinformation2010} and \cite{VeraRoleInformationBottleneck2018}
(see Chapter 3 of \cite{RodriguezGalvezInformationBottleneckConnections2019} for a comprehensive study of the results of these two papers).
The main assumption on which these results are obtained %
is the restricted class assumption stated below. %
\begin{mathdef}[Restricted class \citep{VeraRoleInformationBottleneck2018}]
A random variable $X$ with the alphabet $\mathcal{X}$ is a restricted class random variable if $\mathcal{X}$ has a finite size and there exist $\eta >0$ such that $p(x_{min}):=\min_{x\in\mathcal{X}}p(x)\ge \eta$.
\end{mathdef}

One of the main results of \cite{ShamirLearninggeneralizationinformation2010} is showing %
how the difference between $I(Y;T)$ and its estimation is controlled by $I(X;T)$. This is presented in the next theorem.%
\label{sec:section_including_the_previous_bounds}
\begin{theorem}[Part of Theorem 4 of \cite{ShamirLearninggeneralizationinformation2010}]
\label{thm:shamir4}
Let $X$ and $Y$ be two restricted class random variables with the joint probability $p(x,y)$ and $\delta\in(0,1)$ an arbitrary parameter. Then,  %
with the probability at least $1-\delta$, for all $\hat{X}$ which satisfy the Markov chain property of $Y-X-\hat{X}$ we have
\begin{align}
|I(\hat{X};Y)-&\hat{I}(\hat{X};Y))|\le \nonumber \\
&  \left( C \sqrt{\log (\frac{1}{\delta})  I(X;\hat{X})}\right) \frac{\log n}{\sqrt{n}} + \mathcal{O}(\frac{\log n}{\sqrt{n}}), \nonumber
\end{align}
where $C>\frac{1}{p(x_{\min})}$ is a constant which depends on the distribution $p(x,y)$, and $|\mathcal{X}|$.%
\end{theorem}

This theorem needs %
$I(X;\hat{X})$ which is not known and could be hard to estimate. Recently,  \cite{VeraRoleInformationBottleneck2018} provided a bound for the cross entropy loss which just depends on the empirical term $\hat{I}(X;\hat{X})$. This results is summarized in the following theorem.

\begin{theorem}[Part of Theorem 1 of \cite{VeraRoleInformationBottleneck2018}]
\label{thm:vera1}
Let $X$ and $Y$ be two restricted class random variables with the joint probability $p(x,y)$, then $\forall\delta\in(0,1)$ %
with probability at least $1-\delta$, for all conditional distributions $p(\hat{x}|x)$ 
\begin{align}
\big|\hat{L}_{\text{CE}}\big(p(\hat{x}|x)\big)-&L_{\text{CE}}(p\big(\hat{x}|x\big))\big|\le
\nonumber \\
&  \left( C \sqrt{ \log (\frac{1}{\delta})  \hat{I}(X;\hat{X})}\right) \frac{\log n}{\sqrt{n}} + \mathcal{O}(\frac{\log n}{\sqrt{n}}), \nonumber
\end{align}
where $L_{\text{CE}}$ is the expected cross entropy loss
$L\big(p(\hat{x}|x)\big)=\EX_{X,Y} \left[ -\sum_{\hat{x}} p(\hat{x}|X) \log p(Y|\hat{x}) \right]$ and $C>\frac{1}{p(x_{\min})}$ is a constant which depends on the distribution $p(x,y)$.%
\end{theorem}

In these theorems, the $\mathcal{O}$ notation is used to simplify the results and to focus on the shared parts in %
these bounds that are relevant to the discussion. In the full formulation, there are also other terms including $|\mathcal{X}|$ and $|\mathcal{\hat{X}}|$. Please see \cite{RodriguezGalvezInformationBottleneckConnections2019} for a discussion of these results. In particular, the discussion therein is centered on the dependence of these bounds on $|\hat{\mathcal{X}}|$ and how it makes the bounds hard to achieve in practice where usually the number of possible features is too high (see Example 3.1 of \cite{RodriguezGalvezInformationBottleneckConnections2019}). While this analysis provides a  valid point about the applicability of these results, it is not enough to explain the   %
"bad" generalization of the encoder explained in Example~\ref{exmp:overfitting_encoder},
where $|\hat{\mathcal{X}}|=2$. 
In contrast, we discover that the constant $C$ in these studies, depends on $\frac{1}{p(x_{\min})}$. This observation provides a better insight to the meaning of these bounds. 

To begin the discussion, note that in order to guarantee that the error is less than $\epsilon$ using Theorems~\ref{thm:shamir4} and \ref{thm:vera1}, it is needed that%
\begin{equation}
\label{eq:summary_previous_bounds}
n > C^2 \frac{\log (\frac{1}{\delta})  \hat{I}(X;\hat{X})}{\epsilon^2}.
\end{equation}
In this inequality the logarithmic terms are ignored and therefore a looser lower bound is obtained on the required samples (which is enough for the following discussion).

First of all, note that $C$ has a hidden dependence on the size of the input space because $C>\frac{1}{p(x_{\min})}\ge |\mathcal{X}|$, where the equality holds in the case where the distribution on the input space is uniform. Furthermore, note that $\frac{1}{p}$ is the expected number of trials needed to success in an experiment where the trials are independent and each is successful with probability $p$. As such, $\frac{1}{p(x_{\min})}$ is the expected number of samples needed to observe %
the least probable element in the input space $\mathcal{X}$. Therefore, if one is interested to see all the elements $x\in\mathcal{X}$, an upper bound on the expected number of trials would be $|\mathcal{X}|\times \frac{1}{p(x_{\min})}\le C^2$ samples (i.e. sequentially and independently guarantee that each sample $x_i$ is presented before continuing the trials for the samples $x_j\,;$ $j>i$).
Using these observations, it is easier to interpret Equation~(\ref{eq:summary_previous_bounds}). After these many samples, it is expected that %
even the least probable input pattern is repeated $\frac{1}{\epsilon^2}  \log (\frac{1}{\delta})  \hat{I}(X;\hat{X})$-times. As such, there are many samples for each possible input to estimate the desired values. 

It should be added that this is not surprising since as was %
shown in Example~\ref{exmp:overfitting_encoder}, IB does not put any constraint on the learned models and thus by the "no free lunch" theorem one cannot expect much. To be more precise, this theorem of \cite{Wolpertexistenceprioridistinctions1996} is about the impossibility of \textbf{any} assumption-free guarantees on the Off-Training Set (OTS) error. Being in the assumption-free setting, the only possible way is to make the OTS set small; i.e. guarantee that with a high probability all possible samples are observed. This is exactly what the term $\frac{1}{p(x_{\min})}$ does. Furthermore, in order to achieve good results on the training-set, each of them needs to be observed for at least a few times, which is what the rest of the bound guarantees. It might be possible to improve the bounds by better technical treatments. But, this underlying requirement of "observing almost all of the input patterns" is dictated by the "no free lunch" theorem and cannot be completely eliminated.

\section{Discussion}\label{sec:discussion} %
The presented results have important implications for methods using information compression as a way to improve generalization (such as the information bottleneck). First, the sample complexity bound $2^\frac{H(X)}{\epsilon}$ showed that %
having less entropy, 
exponentially decreases the number of required samples. While this was previously heuristically claimed in the IB literature, to the best of our knowledge, there was no reported theoretical study of the problem in the learning theory settings to achieve the precise results as presented in this paper. Specially, the exponential dependence of the bound on $\frac{1}{\epsilon}$ was not mentioned in the literature which 
has a high influence on the  %
practical usefulness of these bound. 
It was also shown that this dependence is  inherent in the problem 
and it is not possible to %
eliminate this term. This result showed that if $\epsilon$ is small, the application of input compression bounds in their current form is not enough to analyze the generalization properties of methods such as DNNs; i.e. one will still need to find a way to show that the space of learned functions is an strict subset of $\mathcal{F}$. Furthermore, it was discussed that the original form of the IB, in which the features are learned from the dataset, has much worse generalization properties which applies another restriction on the applicability of the IB to analyze DNNs.

These observations also suggest a direction toward fixing these issues. 
Since the main argument in Section~\ref{sec:comparison_with_previous_results}, which showed the poor generalization of the IB, was based on using the labeled dataset, a possible direction is to switch to unsupervised methods for compressed feature extraction purposes. 
While many feature extraction methods are unsupervised (e.g., auto-encoders), application of unsupervised compressed feature extraction in the learning theory has not been analyzed yet and 
is a good direction for future work. %

\begin{appendices}
\section{The Factorized Distribution Assumption}
\label{sec:appendix_factorized_distribution}
As discussed earlier one of the motivations in studying the compression bound, was the idea that if each sample is very large and its distribution is factorized to many parts, one could use the idea of typicality from information theory to find the number of effective functions. In Sec~\ref{sec:problem_with_AEP_arguments}, it was summarized the fallacies in this argument and why it can not be used to derive the generalization bound. 

As this seems to be one of the common pitfalls of using the information theory in practice, most of this section is devoted to dive deeper into these problems. After that we also provide a correct way to %
include this assumption of factorized distribution into the analysis.%

In the rest, for simplicity we assume that the random variable $X$ is the composition of many i.i.d random variables, i.e. $X=(W_1,\dots,W_L)$, and there is a distribution $p(W)$ where 
$\rho(X)=\prod_{i=1}^L p(W_i)$ (it is a special case of factorized distribution which is enough for our next arguments explaining why this additional assumption does not help). In what comes, first it is shown how in first glance it might seem that the $\frac{1}{\epsilon}$ is not needed and after that we show the fallacy %
hidden in the argument.

Typicality is one of the most fundamental concepts in information theory. In its simplest form it can be explained by the simple observation that the law of large numbers can be applied to the random variable $-\frac{1}{L}\log \rho(X)=\frac{1}{L}\sum_{i=1}^L -\log p(W_i)$ to show that it will concentrate around its expected value, i.e. $\EX_W[-\log p(W)]\triangleq H(W)$. More precisely, by weak law of large numbers we can write 
\begin{equation}
\label{eq:AEP_weak_law}
\forall \eta>0; \lim\limits_{L\to\infty}\Pr\{\; |H(W) - \frac{1}{L}\sum_{i=1}^L \widehat{H}(W_i) |>\eta \; \}=0,
\end{equation}
were $\hat{H}(W_i)=-\log p(W_i)$ is a sample from the r.v. we used the law of large numbers to study. 
On the other hand $H(X)=L H(W)$ and by defining $\widehat{H}(X)=-\log \rho(X)=\sum_{i=1}^{L}\widehat{H}(W_i)$ Eq.~(\ref{eq:AEP_weak_law}) can be rewritten as 
\begin{equation}
\label{eq:AEP_weak_law_rewrite}
\forall \eta>0; \lim\limits_{L\to\infty}\Pr\{\; |\widehat{H}(X) - H(X)|>L\eta \; \}=0.
\end{equation}

Now recall that in our analysis, the $\frac{1}{\epsilon}$ factor appeared in proof of Lemma~\ref{lemma:covering_entropy_limited_functions} when bounding the probability of deviation of $\widehat{H}(X)=-\log\rho(X)$ from its expected value. In summary we used Markov inequality 
\begin{equation}
\label{eq:the_markov_ineq_for_entropy}
\Pr\{\; \widehat{H}(X)> \frac{H(X)}{\epsilon} \; \}\le\epsilon; \;\;\epsilon\in(0,1).
\end{equation}
If one could show that in a particular case the l.h.s. is actually zero, 
the $\frac{1}{\epsilon}$ would be completely eliminated from the formulation. 
Actually this might seem to be the case for the factorized distribution as $L\to\infty$. The reason is that the inequality~(\ref{eq:the_markov_ineq_for_entropy}) can be rewritten as 
\begin{align}
\label{eq:rewrite_the_markov_ineq_for_entropy}
\Pr\{\; \widehat{H}(X) -& H(X)> L H(W)(\frac{1}{\epsilon}-1) \; \}\nonumber\\
\le
&\Pr\{\; |\widehat{H}(X) - H(X) | > L (\frac{H(W)(1-\epsilon)}{\epsilon}) \; \},
\end{align}
which by letting $\eta=\frac{H(W)(1-\epsilon)}{\epsilon}$ and using the Eq.~(\ref{eq:AEP_weak_law_rewrite}) approaches $0$ for very large $L$. 

Now 
we demonstrate the error in this line of reasoning. The root of the problem is that in the previous inequalities, $H(X)$ is a function of $L$ and it approaches $\infty$ as $L$ is increased (this would be evident if one use the notation $X_L$ to emphasize the dependence of $\rho(X_L)=\prod_{i=1}^L p(W_i)$ on $L$). In fact we will argue that having $H(X)$ to approach infinity is necessary for a guaranteed concentration. But in the discussion of generalization in previous sections, we were talking about a fixed distribution $\rho(X)$ with finite $H(X)$ (actually we need to have $H(X)\lesssim\log n$ for the generalization bound to be nontrivial, which is $H(X)\le30$ for almost all datasets used in practice today). 

To be more precise, combining inequalities (\ref{eq:AEP_weak_law_rewrite}) and \ref{eq:rewrite_the_markov_ineq_for_entropy}), after substituting the $\lim$ by it's definition we have
\begin{align}
\forall\epsilon\in(0,1); &\forall\delta>0; \exists L_0; \forall L>L_0; 
\nonumber\\
&\Pr\{\; |\widehat{H}(X) - H(X) | > L (\frac{H(W)(1-\epsilon)}{\epsilon}) \; \}<\delta.
\end{align}
The important note in this equation is that  $L_0$ depends on $\epsilon$ (and $\delta$), and thus as $\epsilon$ gets smaller $H(X)$s satisfying the inequality get larger.

One might think that even if $H(X)$ is finite, by just considering that $\widehat{H}(X)$ is the sum of many independent components, there would be some kind of concentration (note that components are non-negative, thus each of them should be very small). To demonstrate that this is not the case, it is insightful to see how Chebyshef inequality, which is the base for one of the proofs of law of large numbers, fails to provide a nontrivial bound. 

By applying Chebyshev's inequality on Eq.~\ref{eq:rewrite_the_markov_ineq_for_entropy}, we get
\begin{align*}
\Pr\left\lbrace\; |\frac{\widehat{H}(X)}{L} - H(W) |\right. >\left.  H(W)(\frac{1-\epsilon}{\epsilon}) \; \right\rbrace
&\le \frac{\frac{1}{L}\sigma_{H(W)}^2}{(H(W) \frac{(1-\epsilon)}{\epsilon})^2}
\\
&=\frac{L\sigma_{H(W)}^2}{H(X)^2 (\frac{(1-\epsilon)}{\epsilon})^2}.
\end{align*}
Now if one considers a finite fixed value of $H(X)$, the bound goes to infinity. This means that we precisely need the $H(X)$ to go to infinity to have concentration, otherwise summing infinitely many independent components which are summed to a finite value does not produce concentration by itself. 

Let's provide a simple concrete example to show that the typicality does not happen. 
Suppose the distribution of each $W_i$ is Bernoulli($p$) and suppose a very large block size of $L=2^{30}$. To have a reasonable value for $H(X)=L H(p)$, the value of $p$ should be small. Let's consider $p=2^{-30}$ which results in having $H(X)\simeq22$. For $X$ to be typical, it means to have the ratio of $1$s close to $p$. In our setting where $p=\slfrac{1}{L}$, we want to see if there is exactly one success in a series of $L$ experiments, where each experiment has the success probability of $\slfrac{1}{L}$. The distribution of number of successes is Binomial($L$,$\frac{1}{L}$) and it is easy to verify that the probability of having more than one success is not zero (in this particular example it is about 0.26). 

The discussion so far reveals that having a factorized distribution does not guarantee a concentration around $H(X)$, thus one have to study the deviation rate. A more interesting question is to ask if this extra assumption can lead to a better bound than the Markov bound of Eq.~(\ref{eq:the_markov_ineq_for_entropy}) and what is the tight bound in this case. This is not an easy question. This problem is related to the general problem of finding "maximal tail probability of sums of nonnegative i.i.d random variables" which is an old problem in probability theory and is still open in the general case 
(\cite{HoeffdingBoundsdistributionfunction1955,Luczakmaximaltailprobability2017}). 

This problem is described as
\begin{equation}
\label{eq:Lth_order_markov} %
m_L(z)\triangleq\sup_Z P(Z_1+\dots +Z_L \ge 1),
\end{equation}
where the supremum is taken over all random vectors $Z=(Z_1,\dots Z_L)$, where $Z_i$s are nonnegative i.i.d variables, such that $\EX[Z_i]\le z$. Note that this is a generalization of Markov inequality where for $L=1$ and $z<1$, we can use the Markov inequality to achieve $M_L(z)=z$ but for $L>1$ this is not always easy.

Now note that by letting $Z_i=\epsilon\slfrac{(-\log p(W_i))}{L H(W)}$, 
the l.h.s of (\ref{eq:the_markov_ineq_for_entropy}) can be written as
\begin{equation}
\Pr\{\; \widehat{H}(X)> \frac{H(X)}{\epsilon} \; \}= \Pr\{\; \sum_{i=1}^L Z_i > 1\;\}.
\end{equation}
Note that here $z=\EX[Z_i] =\frac{\epsilon}{L}$. So in the case where the distribution is factorized to many components (and thus $\widehat{H}(X)$ is some of many nonnegative i.i.d random variables), instead of using the Markov inequality, we can use the specific bound for sum of $L$ random variables  defined in Eq.~(\ref{eq:Lth_order_markov}), i.e.
$$\Pr\{\; \sum_{i=1}^L Z_i > 1\;\} \le m_L(\frac{\epsilon}{L}).$$

As stated before, determining the value of $m_L(z)$ for all $z$ is an open problem, but for the specific range of values which we are interested in, it was recently solved by \cite{Luczakmaximaltailprobability2017}. This is represented in the next theorem.
\begin{theorem}[\cite{Luczakmaximaltailprobability2017}]
For sufficiently large $L$ and for $z\le \frac{1}{2L-1}$ $$m_L(z)=1-(1-z)^L.$$
\end{theorem}

In our discussion where we are interested in very large block sizes, assuming that $\epsilon<\frac{1}{2}$, the conditions of the theorem are met and we can use
$m_L(\frac{\epsilon}{L})=1-(1-\frac{\epsilon}{L})^L$. As $L$ gets larger, this term approaches $1-e^{-\epsilon}$. 

To summarize the result, in case the random variable $X$ is composed of infinitely many components, the Markov inequality in Eq.~(\ref{eq:the_markov_ineq_for_entropy}) can be improved to
\begin{equation}
\label{eq:the_markov_ineq_for_entropy_infinitely_factorized_distribution}
\Pr\{\; \widehat{H}(X)> \frac{H(X)}{\epsilon} \; \}\le 1-e^{-\epsilon}; \;\;\epsilon\in(0,1).
\end{equation}
Using Taylor series around $0$, we can write $1-e^{-\epsilon}=\epsilon+o(\epsilon)$. Thus for small $\epsilon$, our original bound did not improve when adding the additional assumption of factorized distribution. For larger $\epsilon$ there is a small improvement which is presented in figure \ref{fig:simplemarkovvsimprovedboundforfactorizeddistribution}.
\begin{figure}
\centering
\includegraphics[width=0.7\linewidth]{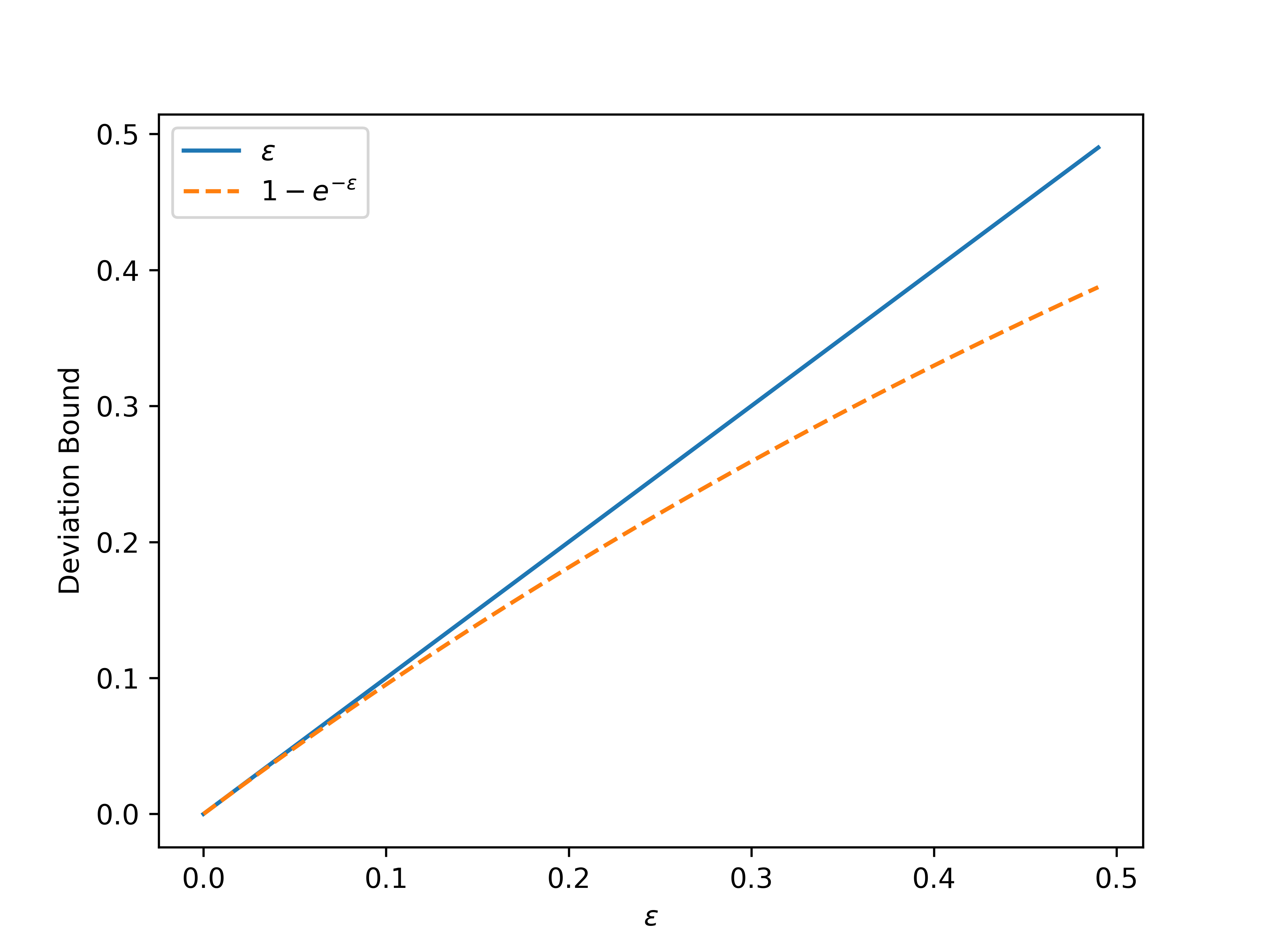}
\caption{The original Markov bound vs the bound achieved when assuming the distribution to be factorized to many parts.}
\label{fig:simplemarkovvsimprovedboundforfactorizeddistribution}
\end{figure}

\end{appendices}

\bibliographystyle{model2-names}
\bibliography{WholeLibrary_BetterBibTex.bib}

\end{document}